\documentclass[11pt,a4paper]{article}
\usepackage[hyperref]{acl2020}
\usepackage{times}
\usepackage{latexsym}

\usepackage{microtype}
\usepackage{amssymb}
\usepackage{amsmath}
\usepackage{blindtext}
\usepackage{amsthm}
\usepackage{mathtools}
\usepackage{booktabs}
\usepackage{multirow}

\aclfinalcopy 
\def\aclpaperid{1900} 

\newcommand{\ignore}[1]{}

\newcommand{\parheader}[1]{{\smallskip \noindent \bf #1.}}

\newcommand{\E}{\mathbb{E}}

\newtheorem{theorem}{Theorem}

\newcommand\dee{\text{d}}
\newcommand\II{\mathbb{I}}
\newcommand{\EE}{\mathbb{E}}

\DeclareMathOperator{\Bias}{Bias}

\DeclareMathOperator{\KS}{KS}

\renewcommand{\proofname}{Proof}

\title{Showing Your Work Doesn't Always Work}

\author{Raphael Tang,$^{1,2}$ Jaejun Lee,$^1$ Ji Xin,$^{1,2}$ Xinyu Liu,$^1$\\ {\bf Yaoliang Yu,$^{1,2}$ \and Jimmy Lin$^{1,2}$}\vspace{0.1cm}\\
  $^1$David R. Cheriton School of Computer Science, University of Waterloo\\
  $^2$Vector Institute for Artificial Intelligence
}

\date{}
\numberwithin{equation}{section}
\begin{document}
\maketitle
\begin{abstract}
In natural language processing, a recently popular line of work explores how to best report the experimental results of neural networks.
One exemplar publication, titled ``Show Your Work:~Improved Reporting of Experimental Results''~\cite{dodge2019show}, advocates for reporting the expected validation effectiveness of the best-tuned model, with respect to the computational budget.
In the present work, we critically examine this paper.
As far as statistical generalizability is concerned, we find unspoken pitfalls and caveats with this approach.
We analytically show that their estimator is biased and uses error-prone assumptions.
We find that the estimator favors negative errors and yields poor bootstrapped confidence intervals.
We derive an unbiased alternative and bolster our claims with empirical evidence from statistical simulation.
Our codebase is at \url{https://github.com/castorini/meanmax}.
\end{abstract}

\section{Introduction}
Questionable answers and irreproducible results represent a formidable beast in natural language processing research.
Worryingly, countless experimental papers lack empirical rigor, disregarding necessities such as the reporting of statistical significance tests~\cite{dror2018guide} and computational environments~\cite{crane2018questionable}.
As \citet{zosa2019scientific} concisely lament, \textit{explorimentation}, the act of tinkering with metaparameters and praying for success, while helpful in brainstorming, does not constitute a rigorous scientific effort.

Against the crashing wave of explorimentation, though, a few brave souls have resisted the urge to feed the beast.
\citet{reimers2017reporting} argue for the reporting of neural network score distributions.
\citet{gorman2019we} demonstrate that deterministic dataset splits yield less robust results than random ones for neural networks.
\citet{dodge2019show} advocate for reporting the expected validation quality as a function of the computation budget used for hyperparameter tuning, which is paramount to robust conclusions.

But carefully tread we must.
Papers that advocate for scientific rigor must be held to the very same standards that they espouse, lest they birth a new beast altogether.
In this work, we critically examine one such paper from \citet{dodge2019show}.
We acknowledge the validity of their technical contribution, but we find several notable caveats, as far as statistical generalizability is concerned.
Analytically, we show that their estimator is negatively biased and uses assumptions that are subject to large errors.
Based on our theoretical results, we hypothesize that this estimator strongly prefers underestimates to overestimates and yields poor confidence intervals with the common bootstrap method~\cite{efron1982jackknife}.

Our main contributions are as follows: First, we prove that their estimator is biased under weak conditions and provide an unbiased solution.
Second, we show that one of their core approximations often contains large errors, leading to poorly controlled bootstrapped confidence intervals.
Finally, we empirically confirm the practical hypothesis using the results of neural networks for document classification and sentiment analysis.

\section{Background and Related Work}

\parheader{Notation}
We describe our notation of fundamental concepts in probability theory.
First, the cumulative distribution function (CDF) of a random variable (RV) $X$ is defined as $F(x) := \Pr[X \leq x]$.
Given a sample $(x_1, \dots, x_B)$ drawn from $F$, the empirical CDF (ECDF) is then $\hat{F}_B(x) := \frac{1}{B}\sum_{i=1}^B \mathbb{I}[x_i \leq x]$,
where $\mathbb{I}$ denotes the indicator function.
Note that we pick ``$B$'' instead of ``$n$'' to be consistent with \citet{dodge2019show}.
The error of the ECDF is popularly characterized by the Kolmogorov--Smirnov (KS) distance between the ECDF and CDF:
\begin{equation}\label{eq:ks}
    \KS(\hat{F}_B, F) := \sup_{x\in \mathbb{R}} |\hat{F}_B(x) - F(x)|.
\end{equation}
Naturally, by definition of the CDF and ECDF, $\KS(\hat{F}_B, F) \leq 1$.
Using the CDF, the expectation for both discrete and continuous (cts.) RVs is
\begin{equation}
    \E[X] = \int^\infty_{-\infty} x \dee F(x),
\end{equation}
defined using the Riemann--Stieltjes integral.

We write the $i^\text{th}$ order statistic of independent and identically distributed (i.i.d.) $X_1, \dots, X_B$ as $X_{(i:B)}$.
Recall that the $i^\text{th}$ order statistic $X_{(i:B)}$ is an RV representing the $i^\text{th}$ smallest value if the RVs were sorted.

\parheader{Hyperparameter tuning}
In random search, a probability distribution $p(\mathcal{H})$ is first defined over a $k$-tuple hyperparameter configuration $\mathcal{H} := (H_1, \dots, H_k)$, which can include both cts.~and discrete variables, such as the learning rate and random seed of the experimental environment.
Commonly, researchers choose the uniform distribution over a bounded support for each hyperparameter~\cite{bergstra2012random}.
Combined with the appropriate model family $\mathcal{M}$ and dataset $\mathcal{D} := (\mathcal{D}_T, \mathcal{D}_V)$---split into training and validation sets, respectively---a configuration then yields a numeric score $V$ on $\mathcal{D}_V$.
Finally, after sampling $B$ i.i.d. configurations, we obtain the scores $V_1, \dots, V_B$ and pick the hyperparameter configuration associated with the best one.

\section{Analysis of Showing Your Work}
In ``Show Your Work: Improved Reporting of Experimental Results,'' \citet{dodge2019show} realize the ramifications of underreporting the hyperparameter tuning policy and its associated budget.
One of their key findings is that, given different computation quotas for hyperparameter tuning, researchers may arrive at drastically different conclusions for the same model.
Given a small tuning budget, a researcher may conclude that a smaller model outperforms a bigger one, while they may reach the opposite conclusion for a larger budget.

To ameliorate this issue, \citet{dodge2019show} argue for fully reporting the expected maximum of the score as a function of the budget.
Concretely, the parameters of interest are $\theta_1, \dots, \theta_B$, where $\theta_n := \E\left[\max\{V_1, \dots, V_n\}\right] = \E[V_{(n:n)}]$
for $1 \leq n \leq B$.
In other words, $\theta_n$ is precisely the expected value of the $n^\text{th}$ order statistic for a sample of size $n$ drawn i.i.d. at tuning time.
For this quantity, they propose an estimator, derived as follows: first, observe that the CDF of $V^*_n = V_{(n:n)}$ is
\begin{align}
    \Pr[V^*_n \leq v] &= \Pr[V_1 \leq v \wedge \cdots \wedge V_n \leq v]\\
    &= \Pr[V \leq v]^n \label{eq:vleqvk},
\end{align}
which we denote as $F^n(v)$. Then
\begin{equation}
    \theta_{n} = \E[V_{(n:n)}] = \int_{-\infty}^{\infty} v \dee F^n(v).
\end{equation}
For approximating the CDF, \citet{dodge2019show} use the ECDF $\hat{F}^n_B(v)$, constructed from some sample $S := (v_1, \dots, v_B)$, i.e.,
\begin{equation}\label{eq:ecdfk}
    \hat{F}^n_B(v) = \left(\hat{F}_B(v)\right)^n = \left(\frac{1}{B} \sum_{i=1}^B \II[v_i \leq v]\right)^n.
\end{equation}
The first identity in Eq.~\eqref{eq:ecdfk} is clear from Eq.~\eqref{eq:vleqvk}.
Without loss of generality, assume $v_1 \leq \cdots \leq v_B$.
To construct an estimator $\hat{\theta}_n$ for $\theta_n$, \citet{dodge2019show} then replace the CDF with the ECDF:
\begin{equation}
    \hat{\theta}_n := \int_{-\infty}^{\infty} v\dee \hat{F}^n_B(v),
\end{equation}
which, by definition, evaluates to
\begin{equation}
    \hat{\theta}_n = \sum_{i=1}^B v_i \left(\hat{F}^n_B(v_i) - \hat{F}^n_B(v_{i-1})\right),
\end{equation}
where, with some abuse of notation, $v_0 < v_1$ is a dummy variable and  $\hat{F}^n_B(v_0) := 0$.
We henceforth refer to $\hat{\theta}_n$ as the \textbf{MeanMax} estimator.
\citet{dodge2019show} recommend plotting the number of trials on the $x$-axis and $\hat{\theta}_n$ on the $y$-axis.

\subsection{Pitfalls and Caveats}

We find two unspoken caveats in \citet{dodge2019show}:~first, the MeanMax estimator is statistically biased, under weak conditions.
Second, the ECDF, as formulated, is a poor drop-in replacement for the true CDF, in the sense that the finite sample error can be unacceptable if certain, realistic conditions are unmet.

\parheader{Estimator bias}
The bias of an estimator $\hat{\theta}$ is defined as the difference between its expectation and its estimand $\theta$: $\Bias(\hat{\theta}) := \E[\hat{\theta}] - \theta$.
An estimator is said to be \textit{unbiased} if its bias is zero; otherwise, it is \textit{biased}.
We make the following claim:

\begin{theorem}\label{thm:bias}
Let $V_1, \dots, V_B$ be an i.i.d. sample (of size $B$) from an unknown distribution $F$ on the real line.
Then, for all $1 \leq n \leq B$,  $\Bias(\hat{\theta}_n) \leq 0$, with strict inequality iff $V_{(1)} < V_{(n)}$ with nonzero probability.
In particular, if $n=1$, then $\Bias(\hat{\theta}_1) = 0$ while if $n > 1$ with $F$ continuous or discrete but non-degenerate, then $\Bias(\hat{\theta}_n) < 0$.
\end{theorem}

\begin{proof}

Let $1 < n \leq B$. We are interested in estimating the expectation of the maximum of the $n$ i.i.d. samples:
\begin{align*}
    \theta_n := \EE [V_{n:n}] = \EE[\max\{V_1, \ldots, V_n\}].
\end{align*}
An obvious unbiased estimator, based on the given sample of size $B$, is the following:
\begin{align*}
    \hat U_n^B := \frac{1}{ {B \choose n} } \sum_{1\leq i_1 < i_2 < \cdots  < i_n  \leq B} \max\{V_{i_1}, \ldots, V_{i_n}\}.
\end{align*}
This estimator is obviously unbiased since 
\begin{align*}
    \EE[\hat U_n^B] = \EE[\max\{V_{i_1}, \ldots, V_{i_n}\}] = \theta_n,
\end{align*}
due to the i.i.d. assumption on the sample. 

A second, biased estimator is the following:
\begin{align}\label{eq:Vn1}
    \hat V_n^B := \frac{1}{ B^n } \sum_{1\leq i_1 \leq i_2 \leq \cdots \leq i_n \leq B} \max\{V_{i_1}, \ldots, V_{i_n}\}.
\end{align}
This estimator is only asymptotically unbiased when $n$ is fixed while $B$ tends to $\infty$. In fact, we will prove below that for all $1 \leq n \leq B$:
\begin{align}
\label{eq:UdV}
    \hat V_n^B \leq \hat U_n^B,
\end{align}
with strict inequality iff $V_{(1)} < V_{(n)}$, where $V_{(i)} = V_{(i:B)}$ is defined as the $i^\text{th}$ smallest order statistic of the sample. 
We start with  simplifying the calculation of the two estimators. 
It is easy to see that the following holds:
\begin{align*}
    \hat U_n^B = \sum_{j=1}^B \frac{ {j-1 \choose n-1} }{ {B \choose n} } V_{(j)},
\end{align*}
where we basically enumerate all possibilities for $\max\{V_{i_1}, \ldots, V_{i_n}\} = V_{(j)}$. By convention, ${m \choose n} = 0$ if $m < n$ so the above summation effectively goes from $k$ to $B$, but our convention will make it more convenient for comparison. Similarly,
\begin{align*}
    \hat V_n^B = \sum_{j=1}^B \frac{ j^n - (j-1)^n }{ B^n } V_{(j)}.
\end{align*}

We make an important observation that connects our estimators to that of Dodge et al. Let $\hat F_B(x) = \frac{1}{B} \sum_{i=1}^B \II[V_i \leq x]$ be the empirical distribution of the sample. Then, the plug-in estimator, where we replace $F$ with $\hat F_B$, is
\begin{align*}
    \hat\theta_n^B &= \hat\EE[\max\{\hat V_1, \ldots, \hat V_n\}], \quad \mbox{ where } \quad \hat V_i \stackrel{iid}{\sim} \hat F_B \\
    &= \sum_{j=1}^B [\hat F_B^n(V_{(j)}) - \hat F_B^n(V_{(j-1)})] V_{(j)}
    = \hat V_n^B
    ,
\end{align*}
since $\hat F_B^n(V_{(j)}) = (j/B)^n$ if there are no ties in the sample. The formula continues to hold even if there are ties, in which case we simply collapse the ties, using the fact that $\sum_{j=i}^k \hat F_B^n(V_{(j)}) - \hat F_B^n(V_{(j-1)}) = \hat F_B^n(V_{(k)}) - \hat F_B^n(V_{(i-1)})$ when $V_{(i-1)} < V_{(i)} = V_{(i+1)} = \cdots = V_{(k)} < V_{(k+1)}.$

Now, we are ready to prove Eq.~\eqref{eq:UdV}. All we need to do is to compare the cumulative sums of the coefficients in the two estimators:
\begin{align*}
    \sum_{j=1}^k \frac{ {j-1 \choose n-1} }{ {B \choose n} } = \frac{ {k \choose n} }{ {B \choose n} }, \quad 
    \sum_{j=1}^k \frac{ j^n - (j-1)^n }{ B^n } = \frac{k^n}{B^n}.
\end{align*}
We need only consider $k \geq n$ (the case $k< n$ is trivial). One can easily verify the following expression backwards:
\begin{align*}
    \frac{ {k \choose n} }{ {B \choose n} } < \frac{k^n}{B^n} &\iff \frac{{k \choose n}}{k^n} < \frac{ {B \choose n} }{ B^n }\\ &\iff \prod_{i=0}^{n-1} (1-\frac{i}{k}) < \prod_{i=0}^{n-1} (1-\frac{i}{B}),
\end{align*}
where the last inequality follows from $k < B$ and $n > 1$.
Thus, we have verified the following for all $1 \leq k < B$:
\begin{align*}
    \sum_{j=1}^k \frac{ {j-1 \choose n-1} }{ {B \choose n} } < \sum_{j=1}^k \frac{ j^n - (j-1)^n }{ B^n }.
\end{align*}
Eq.~\eqref{eq:UdV} now follows since $V_{(1)} < \cdots < V_{(B)}$ lies in the isotonic cone while we have proved the difference of the two coefficients lies in the dual cone of the isotonic cone. 
An elementary way to see this is to first compare the coefficients in front of $V_{(B)}$: clearly, $\hat U_n^B$'s is larger since it has smaller sum of all coefficients (but the one in front of $V_{(B)}$; take $k=B-1$) whereas the total sum is always one.
Repeat this comparison for $V_{(1)}, \dots, V_{(B-1)}$.

Lastly, if $V_{(1)} < V_{(n)}$, then there exists a subset (with repetition) $1 \leq i_1 \leq \ldots \leq i_n \leq n $ such that $\max\{ V_{(i_1)}, \ldots, V_{(i_n)} \} < V_{(n)}$. For instance, setting $i_1 = \ldots = i_n = 1$ would suffice. Since $\hat V_n^B$ puts positive mass on every subset of $n$ elements (with repetitions allowed), the strict inequality follows.
We note that if $F$ is continuous, or if $F$ is discrete but non-degenerate, then $V_{(1)} < V_{(n)}$ with nonzero probability, hence
\begin{align*}
    \Bias(\hat\theta_n) = \EE(\hat V_n^B - \hat U_n^B) < 0.
\end{align*}
The proof is now complete.
\end{proof}

For further caveats, see Appendix A.
The practical implication is that researchers may falsely conclude, on average, that a method is worse than it is, since the MeanMax estimator is negatively biased.
In the context of environmental consciousness~\cite{schwartz2019green}, more computation than necessary is used to make a conclusion.
\renewcommand{\proofname}{Proof}

\parheader{ECDF error}
The finite sample error (Eq.~\ref{eq:ks}) of approximating the CDF with the ECDF (Eq.~\ref{eq:ecdfk}) can become unacceptable as $n$ increases:

\begin{theorem}
If the sample does not contain the population maximum, $\KS(\hat{F}^n_B, F^n) \rightarrow 1$ exponentially quickly as $n$ and $B$ increase.
\end{theorem}

\begin{proof}
See Appendix B.
\end{proof}

Notably, this result always holds for cts.~distributions, since the population maximum is never in the sample.
Practically, this theorem suggests the failure of bootstrapping~\cite{efron1982jackknife} for statistical hypothesis testing and constructing confidence intervals (CIs) of the expected maximum, since the bootstrap requires a good approximation of the CDF~\cite{canty2006bootstrap}.
Thus, relying on the bootstrap method for constructing confidence intervals of the expected maximum, as in \citet{lucic2018gans}, may lead to poor coverage of the true parameter.

\section{Experiments}

\subsection{Experimental Setup}
To support the validity of our conclusions, we opt for cleanroom Monte Carlo simulations, which enable us to determine the true parameter and draw millions of samples.
To maintain the realism of our study, we apply kernel density estimation to actual results, using the resulting probability density (or discretized mass) function as the ground truth distribution.
Specifically, we examine the experimental results of the following neural networks:

\parheader{Document classification}
We first conduct hyperparameter search over neural networks for document classification, namely a multilayer perceptron (MLP) and a long short-term memory (LSTM; \citealp{hochreiter1997long}) model representing state of the art (for LSTMs) from~\citet{adhikari2019rethinking}.
For our dataset and evaluation metric, we choose Reuters~\cite{apte1994automated} and the F$_1$ score, respectively.
Next, we fit discretized kernel density estimators to the results---see the appendix for experimental details.
We name the distributions after their models, MLP and LSTM.

\parheader{Sentiment analysis}
Similar to \citet{dodge2019show}, on the task of sentiment analysis, we tune the hyperparameters of two LSTMs---one ingesting embeddings from language models (ELMo; \citealp{peters2018deep}), the other shallow word vectors~(GloVe; \citealp{pennington2014glove}).
We choose the binary Stanford Sentiment Treebank~\cite{socher2013recursive} dataset and apply the same kernel density estimation method.
We denote the distributions by their embedding types, GloVe and ELMo.

\subsection{Experimental Test Battery}

\parheader{False conclusion probing}
To assess the impact of the estimator bias, we measure the probability of researchers falsely concluding that one method underperforms its true value for a given $n$.
The \mbox{\textit{unbiased}} estimator has an expectation of $0.5$, preferring neither underestimates nor overestimates.

Concretely, denote the true $n$-run expected maxima of the method as $\theta_n$ and the estimator as $\hat{\theta}_n$.
We iterate $n=1,\dots,50$ and report the proportion of samples (of size $B=50$) where $\hat{\theta}_n < \theta_n$. We compute the true parameter using 1,000,000 iterations of Monte Carlo simulation and estimate the proportion with 5,000 samples for each $n$.

\parheader{CI coverage}
To evaluate the validity of bootstrapping the expected maximum, we measure the coverage probability of CIs constructed using the percentile bootstrap method~\cite{efron1982jackknife}.
Specifically, we set $B=50$ and iterate $n = 1, \dots, 50$.
For each $n$, across $M=1000$ samples, we compare the empirical coverage probability (ECP) to the nominal coverage rate of 95$\%$, with CIs constructed using $5,000$ bootstrapped resamples.
The ECP $\hat{\alpha}_n$ is computed as
\begin{equation}
    \hat{\alpha}_n := \frac{1}{M} \sum_{i=1}^M \mathbb{I}\left(\theta_n \in \text{CI}_i\right),
\end{equation}
where CI$_i$ is the CI of the $i^\text{th}$ sample.
\subsection{Results}
\begin{figure}[t]
    \centering
    \includegraphics[scale=0.26,trim={0.45cm 0 0.45cm 0},clip]{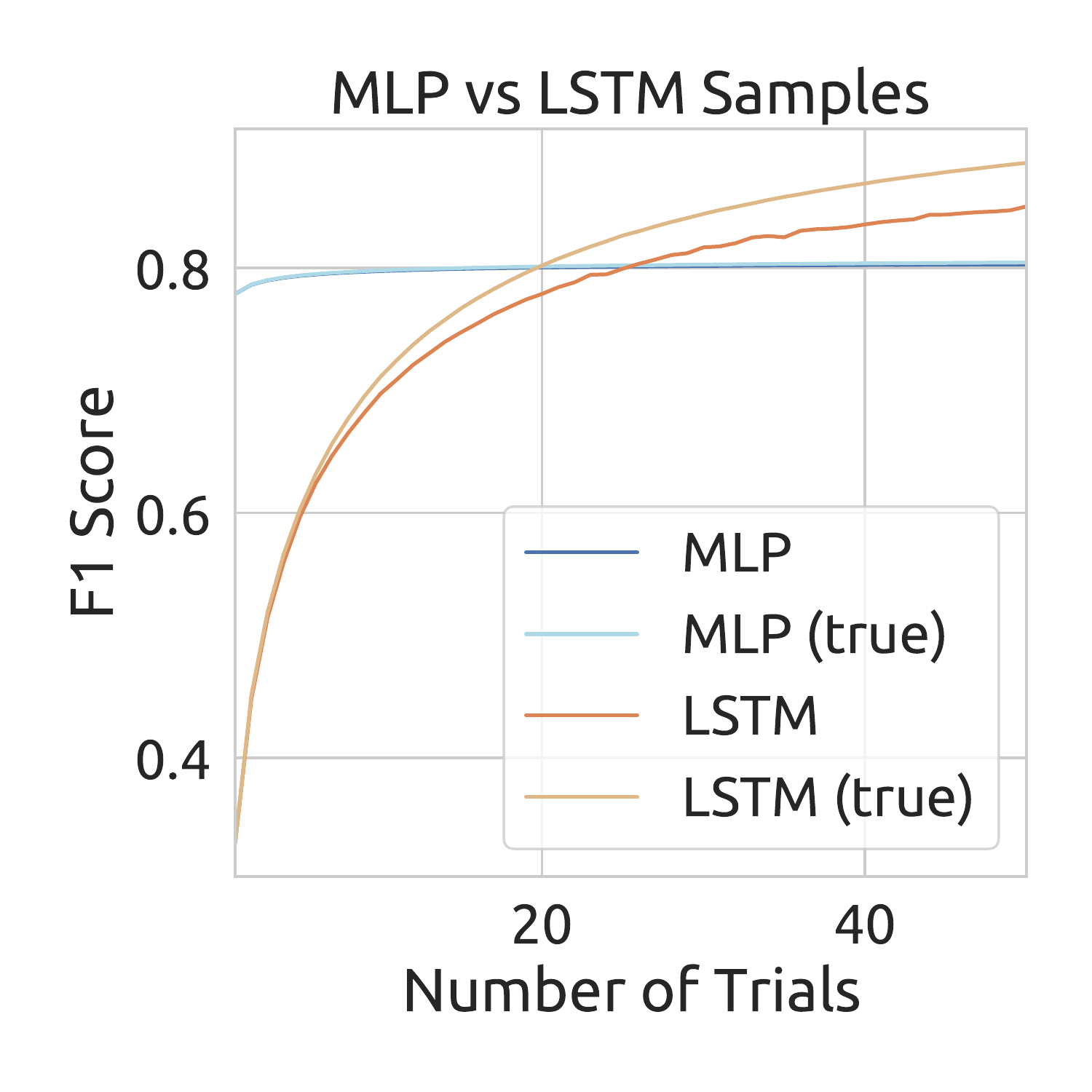}
    \includegraphics[scale=0.26,trim={0.45cm 0 0.45cm 0},clip]{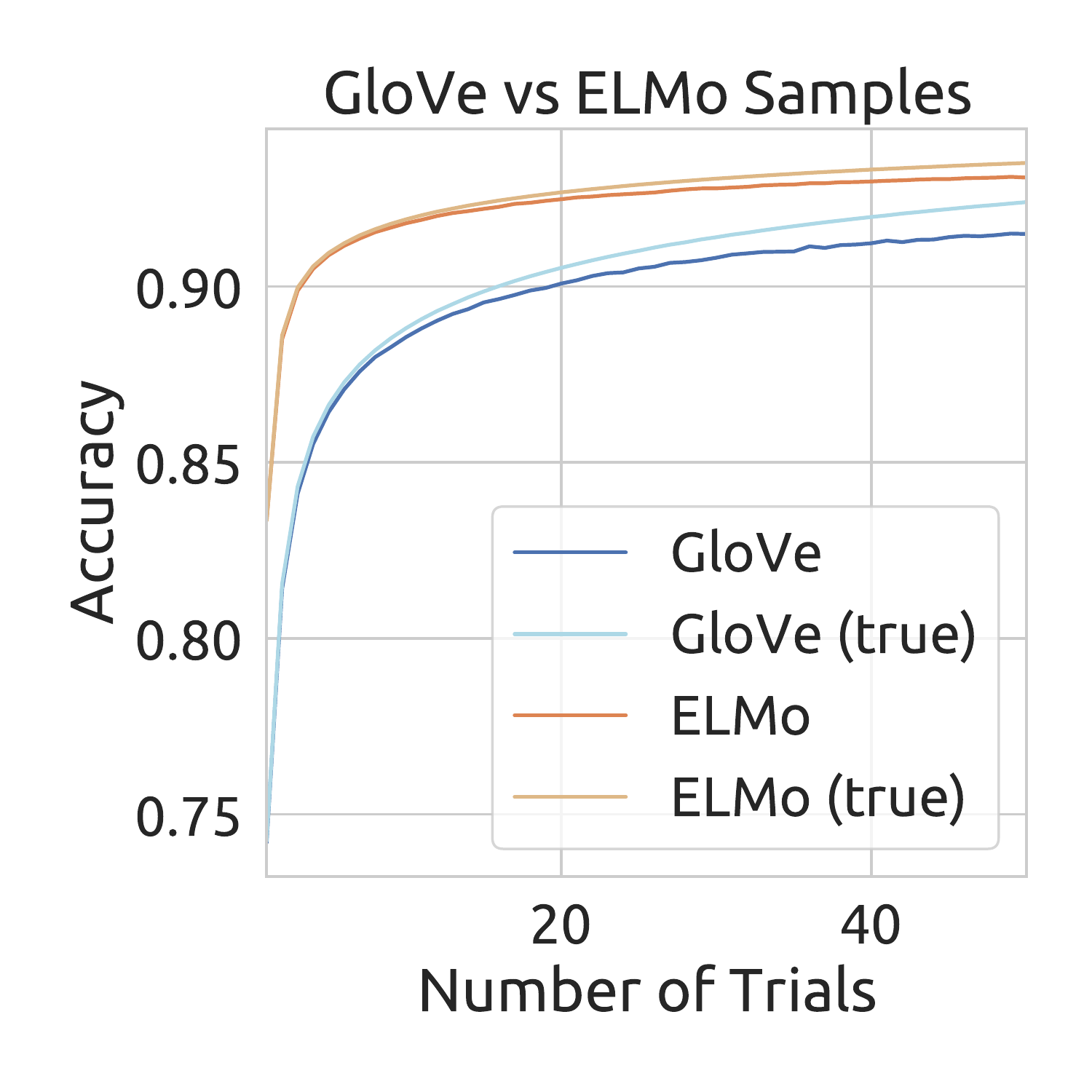}
    \caption{The estimated budget--quality curves, along with the true curves.}
    \label{fig:curve}
\end{figure}
\begin{figure}[t]
    \centering
    \includegraphics[scale=0.26,trim={0.45cm 0 0.45cm 0},clip]{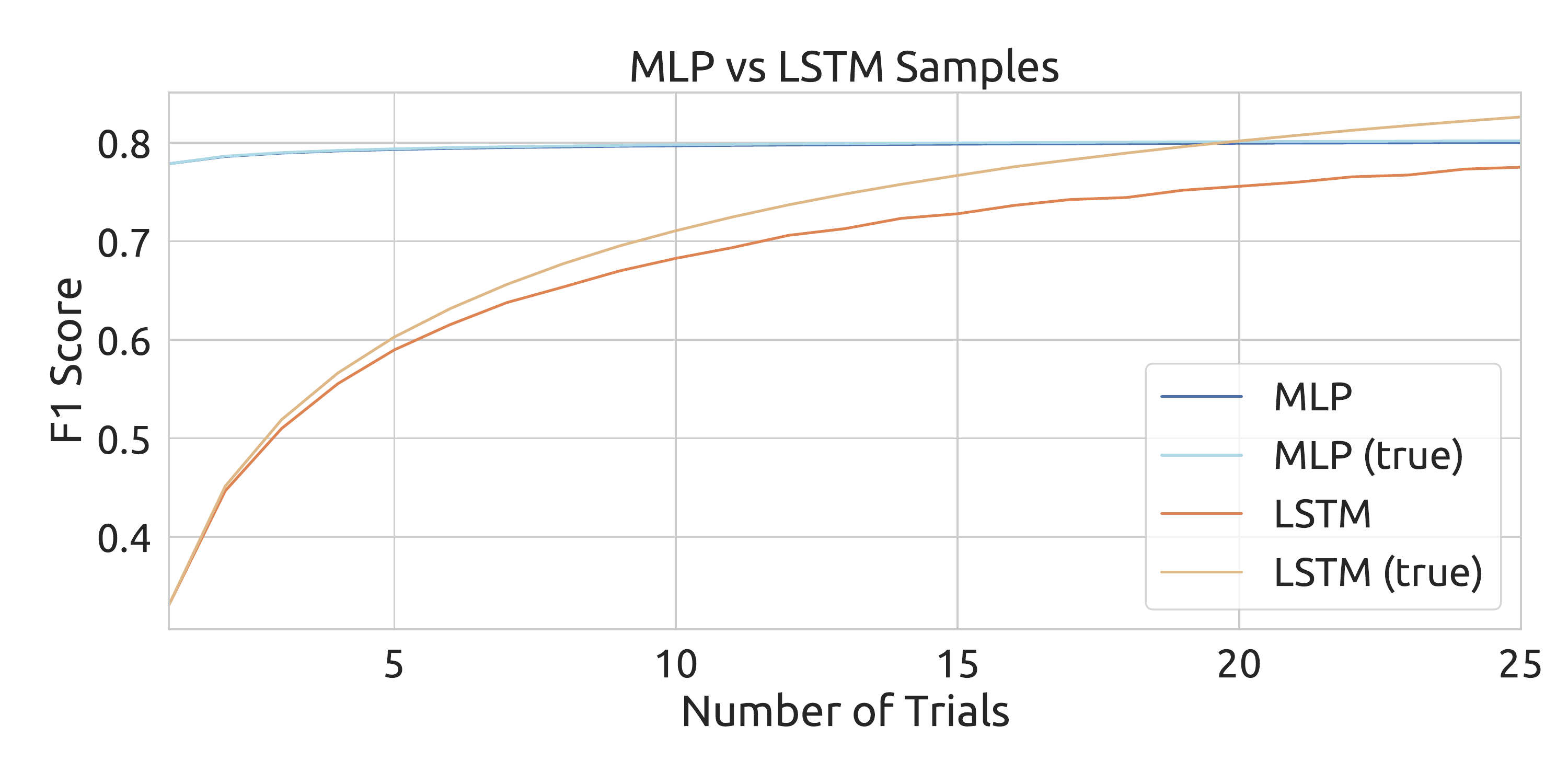}
    \caption{Illustration of a failure case with $B=25$.}
    \label{fig:failcase}
\end{figure}
Following \citet{dodge2019show}, we present the budget--quality curves for each model pair in Figure~\ref{fig:curve}.
For each $n$ number of trials, we vertically average each curve across the 5,000 samples.
We construct CIs but do not display them, since the estimate is precise (standard error $< 0.001$).
For document classification, we observe that the LSTM is more difficult to tune but achieves higher quality after some effort.
For sentiment analysis, using ELMo consistently attains better accuracy with the same number of trials---we do not consider the wall clock time.

In Figure~\ref{fig:failcase}, we show a failure case of biased estimation in the document classification task.
At $B=25$, from $n=20$ to $25$, the averaged estimate yields the wrong conclusion that the MLP outperforms the LSTM---see the true LSTM line, which is above the true MLP line, compared to its estimate, which is below.

\begin{figure}[t]
    \centering
    \includegraphics[scale=0.25,trim={0.45cm 0 0.45cm 0},clip]{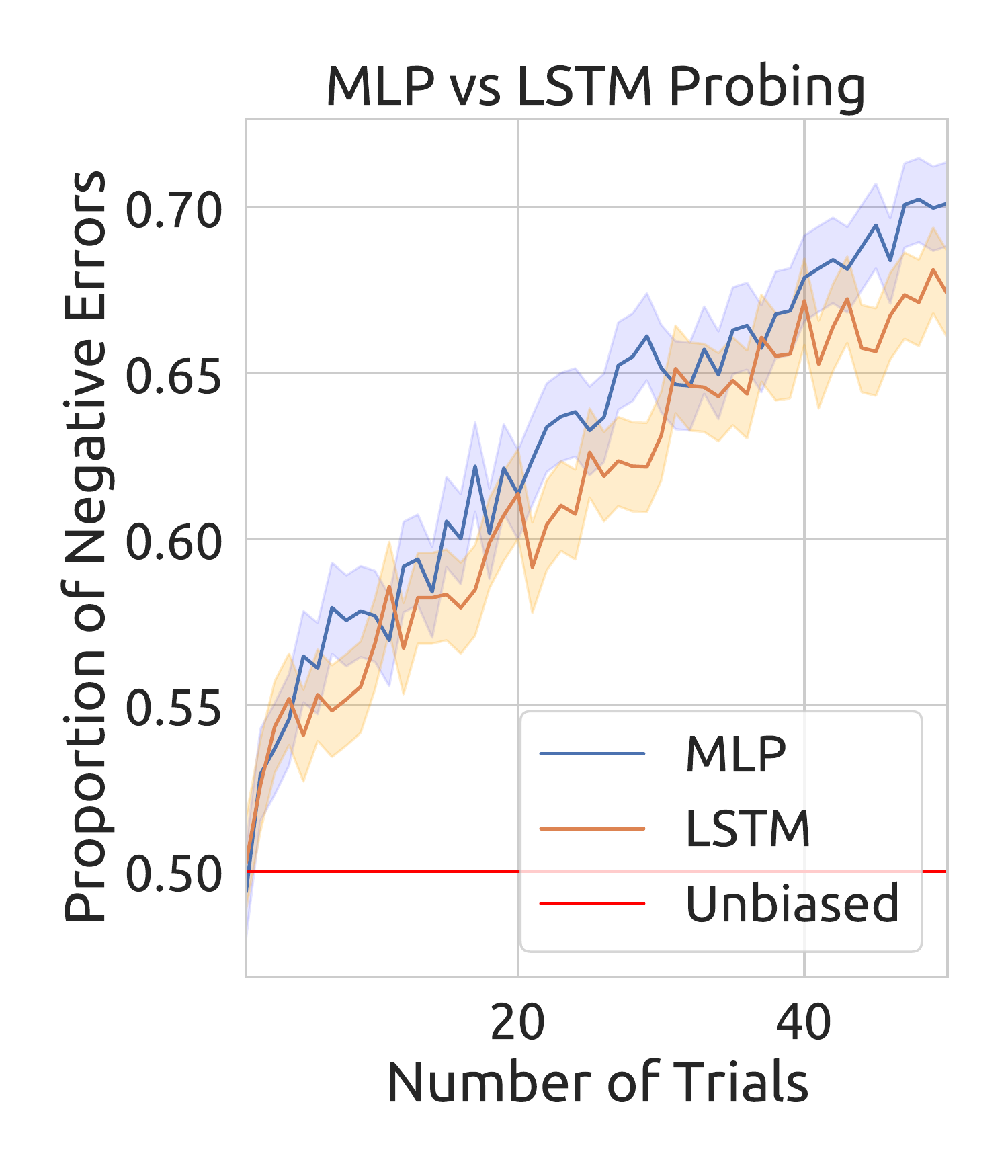}
    \includegraphics[scale=0.25,trim={0.45cm 0 0.45cm 0},clip]{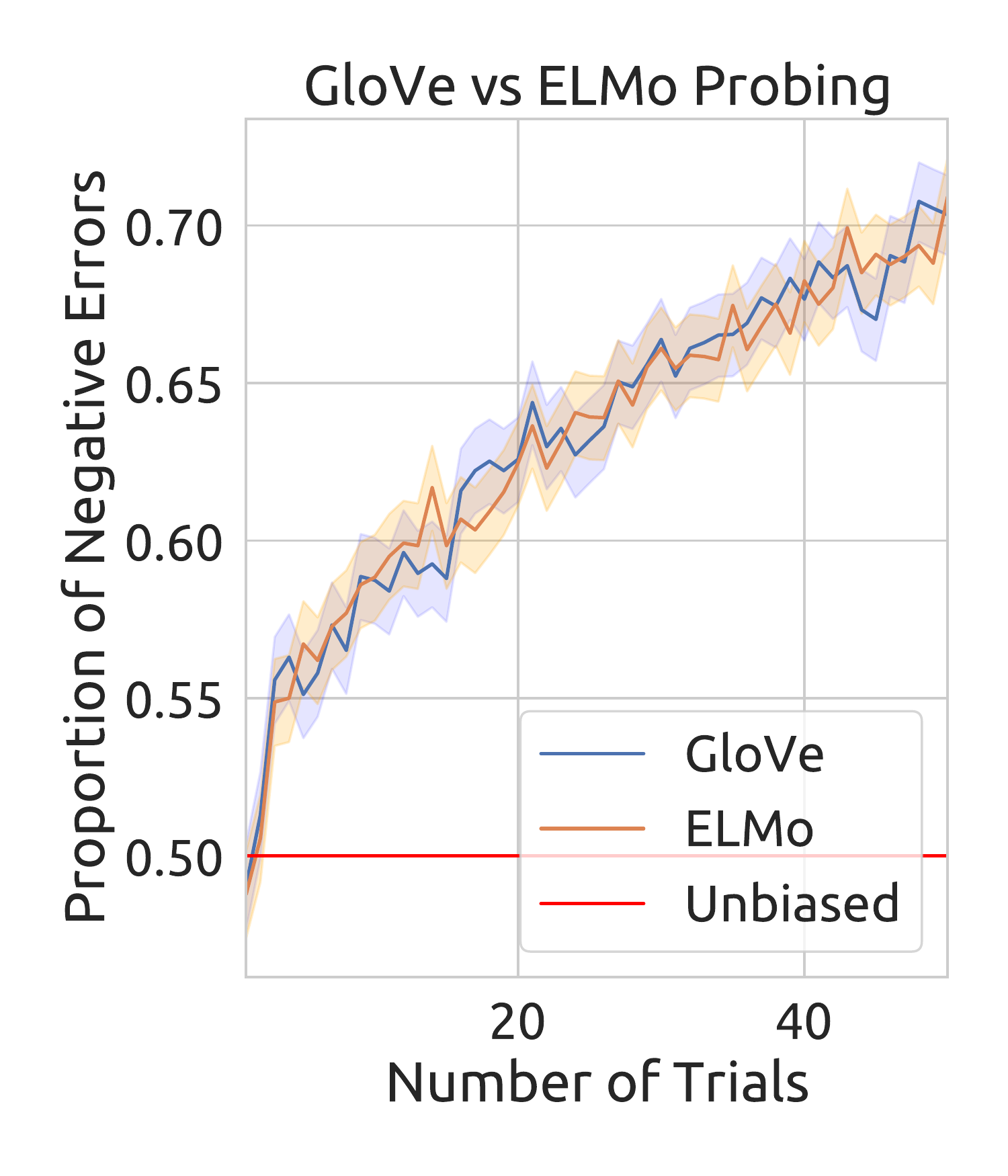}
    \caption{The false conclusion probing experiment results, along with Clopper--Pearson 95\% CIs.}
    \label{fig:exp1}
\end{figure}
\begin{figure}[t]
    \centering
    \includegraphics[scale=0.25,trim={0.45cm 0 0.45cm 0},clip]{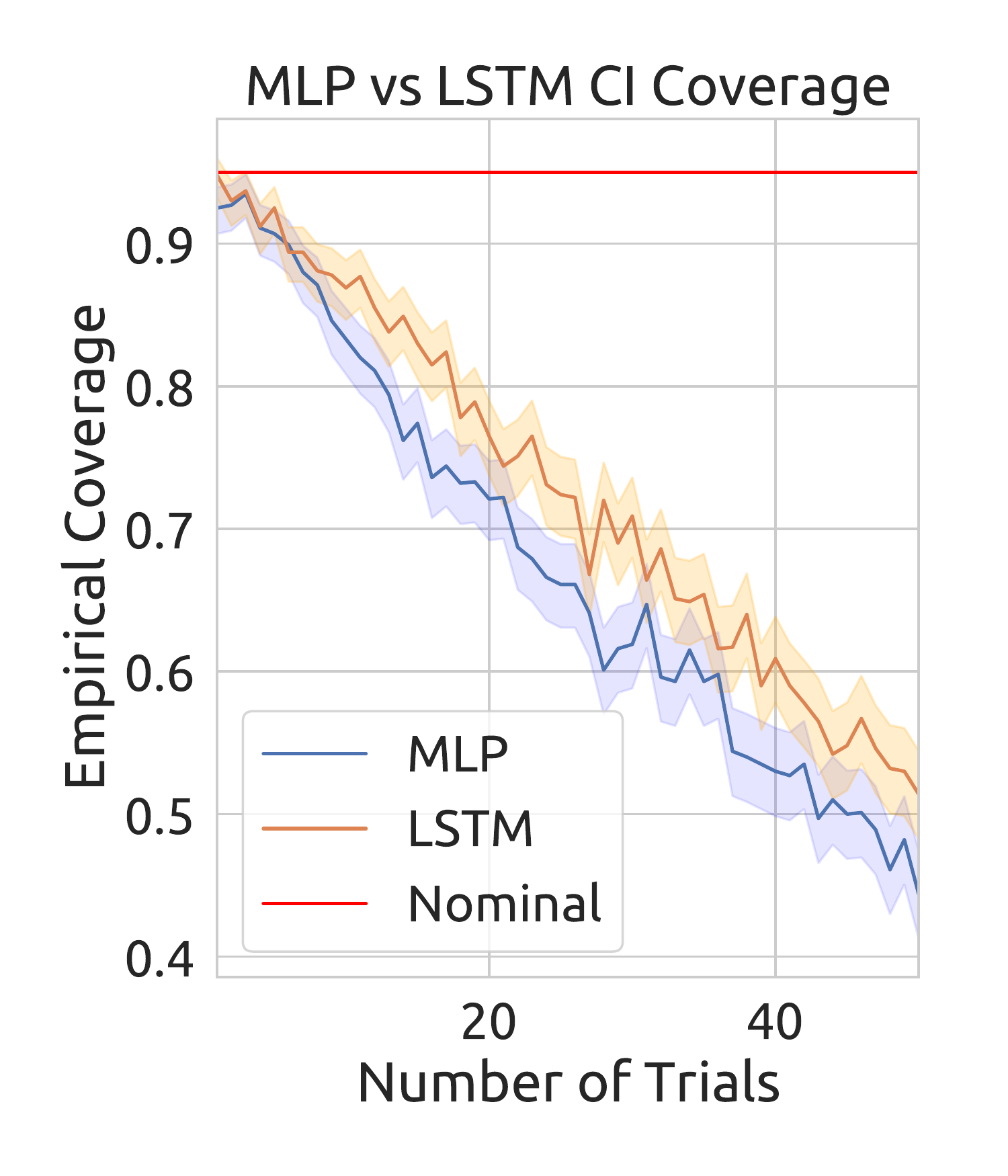}
    \includegraphics[scale=0.25,trim={0.45cm 0 0.45cm 0},clip]{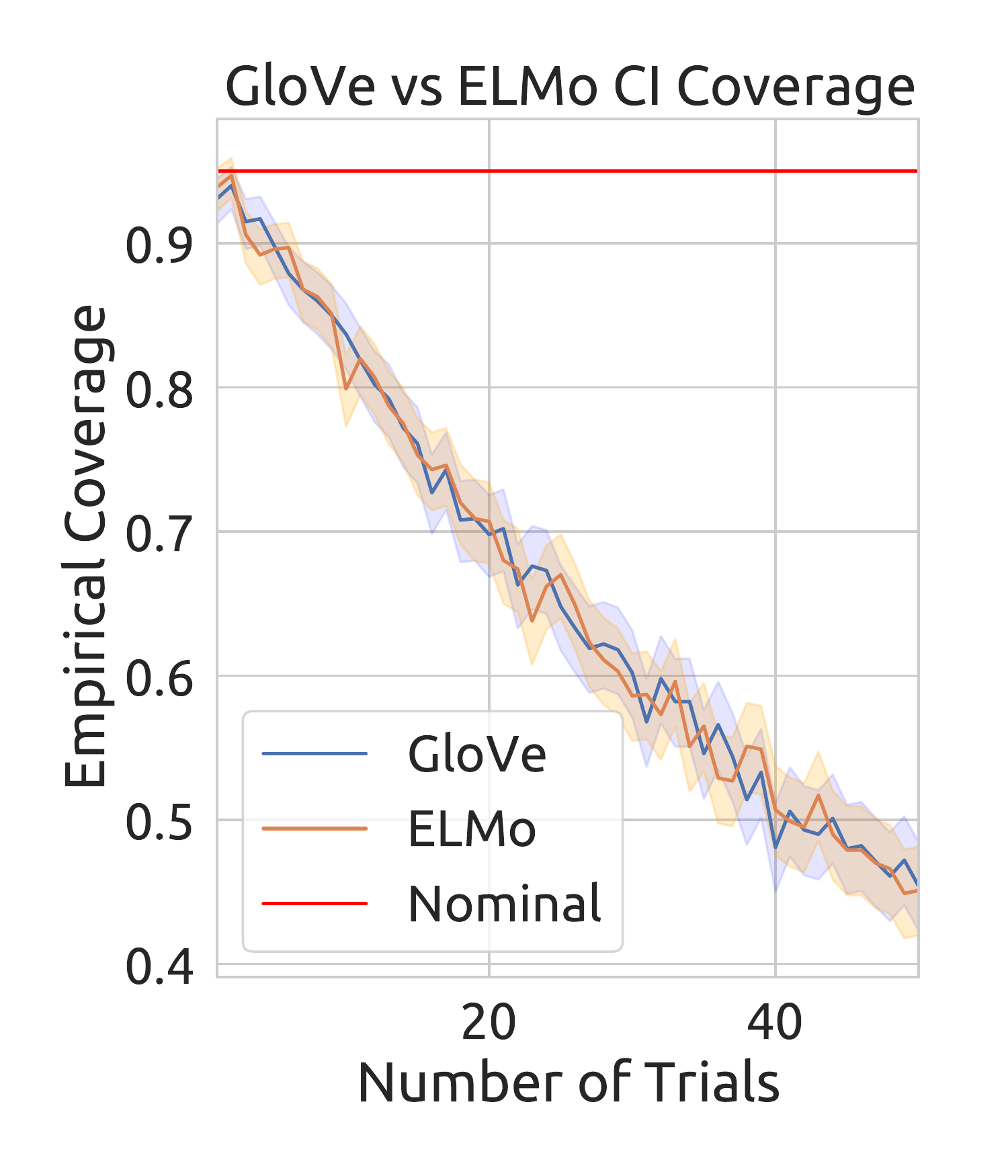}
    \caption{The CI coverage experiment results, along with Clopper--Pearson 95\% CIs.}
    \label{fig:exp2}
\end{figure}

\parheader{False conclusions probing}
Figure~\ref{fig:exp1} shows the results of our false conclusion probing experiment.
We find that the estimator quickly prefers negative errors as $n$ increases.
The curves are mostly similar for both tasks, except the MLP fares worse.
This requires further analysis, though we conjecture that the reason is lower estimator variance, which would result in more consistent errors.

\parheader{CI coverage}
We present the results of the CI coverage experiment results in Figure~\ref{fig:exp2}.
We find that the bootstrapped confidence intervals quickly fail to contain the true parameter at the nominal coverage rate of $0.95$, decreasing to an ECP of $0.7$ by $n=20$.
Since the underlying ECDF is the same, this result extends to \citet{lucic2018gans}, who construct CIs for the expected maximum.

\section{Conclusions}
In this work, we provide a dual-pronged theoretical and empirical analysis of \citet{dodge2019show}.
We find unspoken caveats in their work---namely, that the estimator is statistically biased under weak conditions and uses an ECDF assumption that is subject to large errors.
We empirically study its practical effects on tasks in document classification and sentiment analysis.
We demonstrate that it prefers negative errors and that bootstrapping leads to poorly controlled confidence intervals.

\section*{Acknowledgments}
This research was supported by the Natural Sciences and Engineering Research Council (NSERC) of Canada.

\bibliography{main-sigtest}
\bibliographystyle{acl_natbib}
\begin{table*}[t]
    \centering
    \scriptsize
    \setlength{\tabcolsep}{1pt}
    \begin{tabular}{lcccccccccc}
    \toprule[1pt]
         Model & Mode & Batch Size & Learning Rate & Seed & Dropout & \# Layers & Hidden Dim. & WDrop & EDrop & $\beta_\text{EMA}$\\
         \midrule
         MLP & -- & $(16, 32, 64)$ & $0.001$ & $[0,10^7]_D$ & $[0.05,0.7]$ & 1 & $[256, 768]_D$ & -- & -- & -- \\
         \multirow{2}{*}{LSTM} & \multirow{2}{*}{\shortstack{$(\textrm{nonstatic}[0.5],$\\$\textrm{static}[0.4],\textrm{rand}[0.1])$}} & \multirow{2}{*}{$(16,32,64)$} & \multirow{2}{*}{$\textrm{TExp}[0.001,0.099]$} & \multirow{2}{*}{$[0,10^7]_D$} & \multirow{2}{*}{$[0.05,0.7]$} & \multirow{2}{*}{$(1[0.75],2[0.25])$} &
         \multirow{2}{*}{$[384,768]_D$} &
         \multirow{2}{*}{$[0, 0.3]$} &
         \multirow{2}{*}{$[0,0.3]$} &
         \multirow{2}{*}{$[0.985,0.995]$}\\
         \\
    \bottomrule[1pt]
    \end{tabular}\vspace{2mm}
    \caption{Hyperparameter random search bounds. $[\cdot,\cdot]_D$ indicates a discrete uniform range, while $[\cdot,\cdot]$ continuous uniform. \textsc{TExp}$[\cdot, \cdot]$ denotes the truncated exponential distribution. Tuples represent categorical distributions, uniform by default. WDrop and EDrop denote weight and embed dropout. For the GloVe- and ELMo-based search bounds, see \url{https://github.com/allenai/show-your-work}.}
    \label{tab:hyperparams}
\end{table*}
\newpage
\appendix
\section{Cautionary Notes}

We caution that the estimator described in \textit{the text} of Dodge et al. is $\hat V_n^n$.
This is clear from their equation (7) where the empirical distribution is defined over the first $n$ samples, instead of the $B$ samples that we use here.
In other words, they claim, at least in the text, to use $\hat F_n$ instead of $\hat F_B$ for their estimator $\hat V_n^n$.
Clearly, the estimator $\hat V_n^n$ is (much) worse than $\hat V_n^B$ since the latter exploits all $B$ samples while the former only looks at the first $n$ samples.
However, close examination of their codebase\footnote{\url{https://github.com/allenai/allentune}} reveals that they use $\hat{V}^B_n$, so the paper discrepancy is a simple notation error.

Lastly, we mention that our notation for $\hat U_n^B$ and $\hat V_n^B$ is motivated by the fact that the former is a $U$-statistic while the latter is a $V$-statistic. The relation between the two has been heavily studied in statistics since Hoeffding's seminar work. For us, it suffices to point out that $\hat V_n^B \leq \hat U_n^B$, with the latter being unbiased while the former is only asymptotically unbiased.
The difference between the two is more pronounced when $n$ is close to $B$.
We note that $\hat U_n^B$ can be computed by a reasonable approximation of the binomial coefficients, using say Stirling's formula.

\section{Proof of Theorem 2}
\begin{theorem}
If the sample does not contain the population maximum, $\KS(\hat{F}^n_B, F^n) \rightarrow 1$ exponentially quickly as $n$ and $B$ increase.
\end{theorem}
	\begin{proof}
	Suppose $v^*$ is not in the sample $v_1, \dots, v_B$, where $v_1 \leq \cdots \leq v_B < v^*$.
Then
\begin{align*}
    \sup_{x\in \mathbb{R}} |\hat{F}^n_B(x) - F^n(x)| &\geq |\hat{F}^n_B(v_B) - F^n(v_B)|.
\end{align*}
From Equation 2.1, $\hat{F}^n_B(v_B) = (\hat{F}_B(v_B))^n = 1 > (F(v_B))^n = F^n(v_B)$, hence
\begin{equation*}
    |\hat{F}^n_B(v_B) - F^n(v_B)| = 1 - (F(v_B))^n.
\end{equation*}
Thus concluding the proof.
\end{proof}

\section{Experimental Settings}
\begin{table}[t]
    \centering
    \setlength{\tabcolsep}{2pt}
    \begin{tabular}{lcccc}
    \toprule[1pt]
         Model & \# Runs & Bandwidth & Support & Bins\\
         \midrule
         MLP & $145$ & $0.0049$ & $[0.72, 0.82]$ & 511\\
         LSTM & $152$ & $0.059$ & $[-0.18, 1.08]$ & 511\\
         GloVe & $114$ & $0.018$ & $[0.46, 0.97]$ & 511\\
         ELMo & $84$ & $0.041$ & $[0.39, 0.99]$ & 511\\
    \bottomrule[1pt]
    \end{tabular}
    \caption{Model kernel parameters. Bandwidth chosen using Scott's normal reference rule. Bins denote the number of discretized slots.}
    \label{tab:kparams}
\end{table}
\begin{figure}[t]
    \centering
    \includegraphics[scale=0.23]{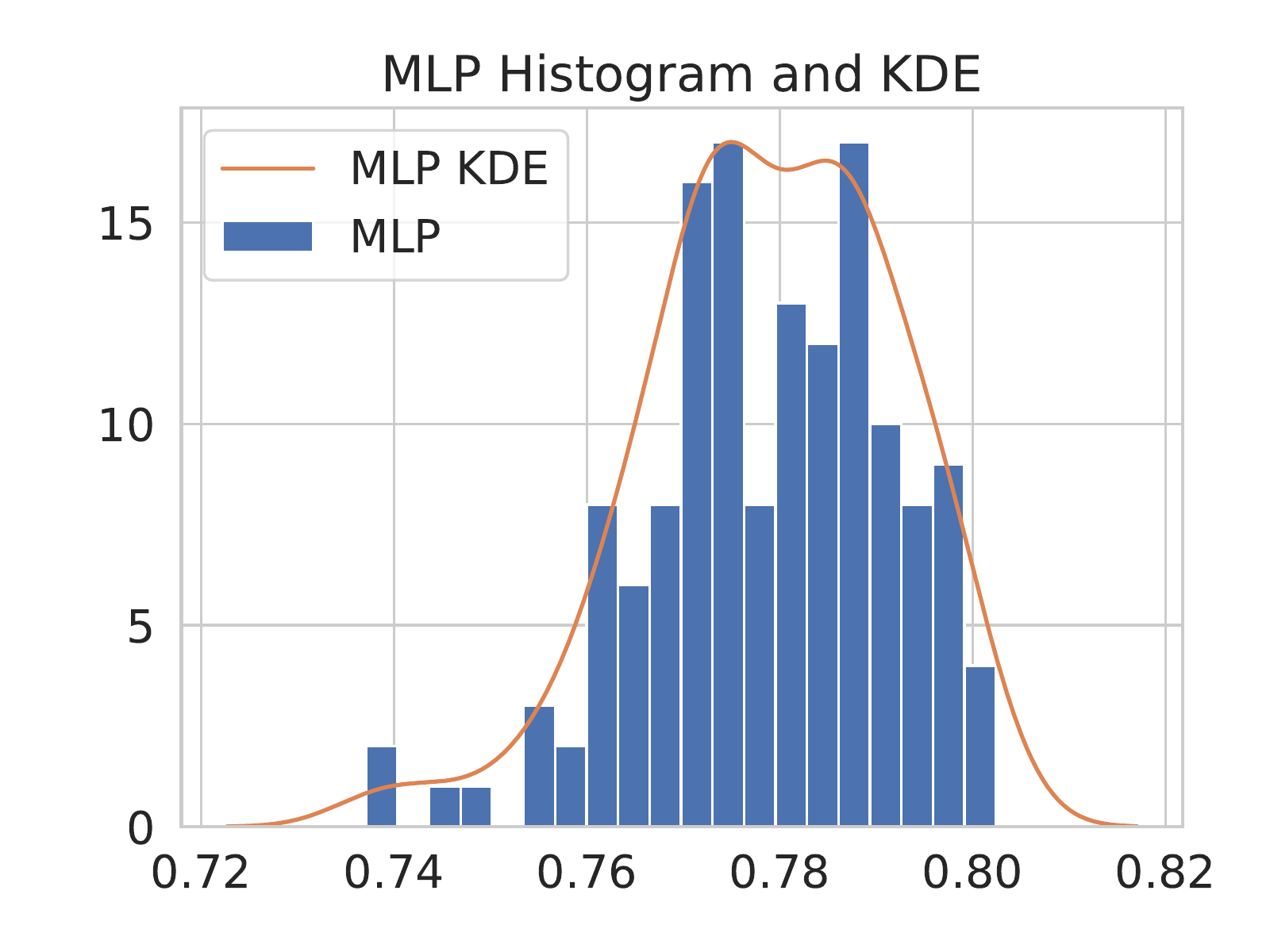}
    \includegraphics[scale=0.23]{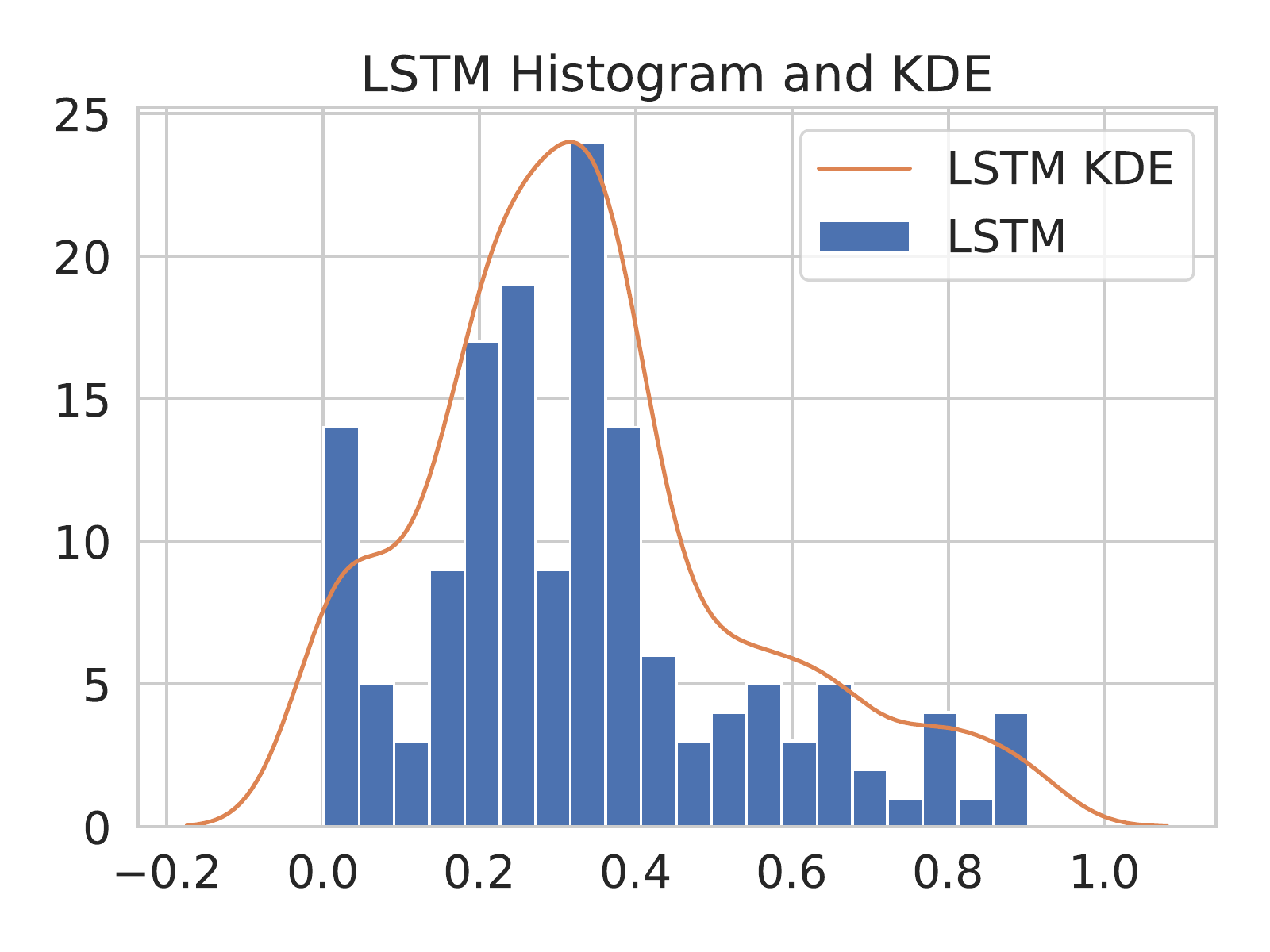}
    \includegraphics[scale=0.23]{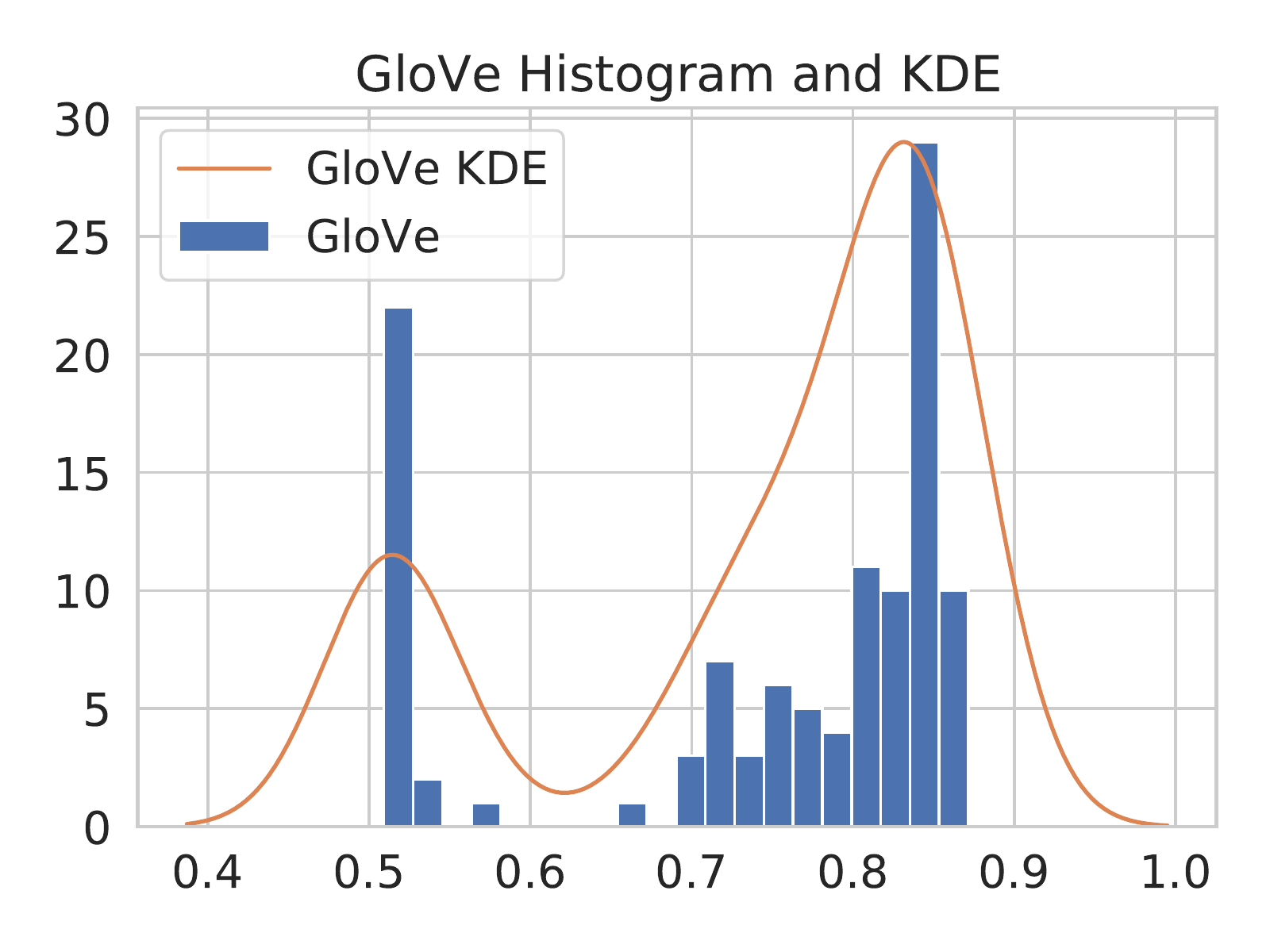}
    \includegraphics[scale=0.23]{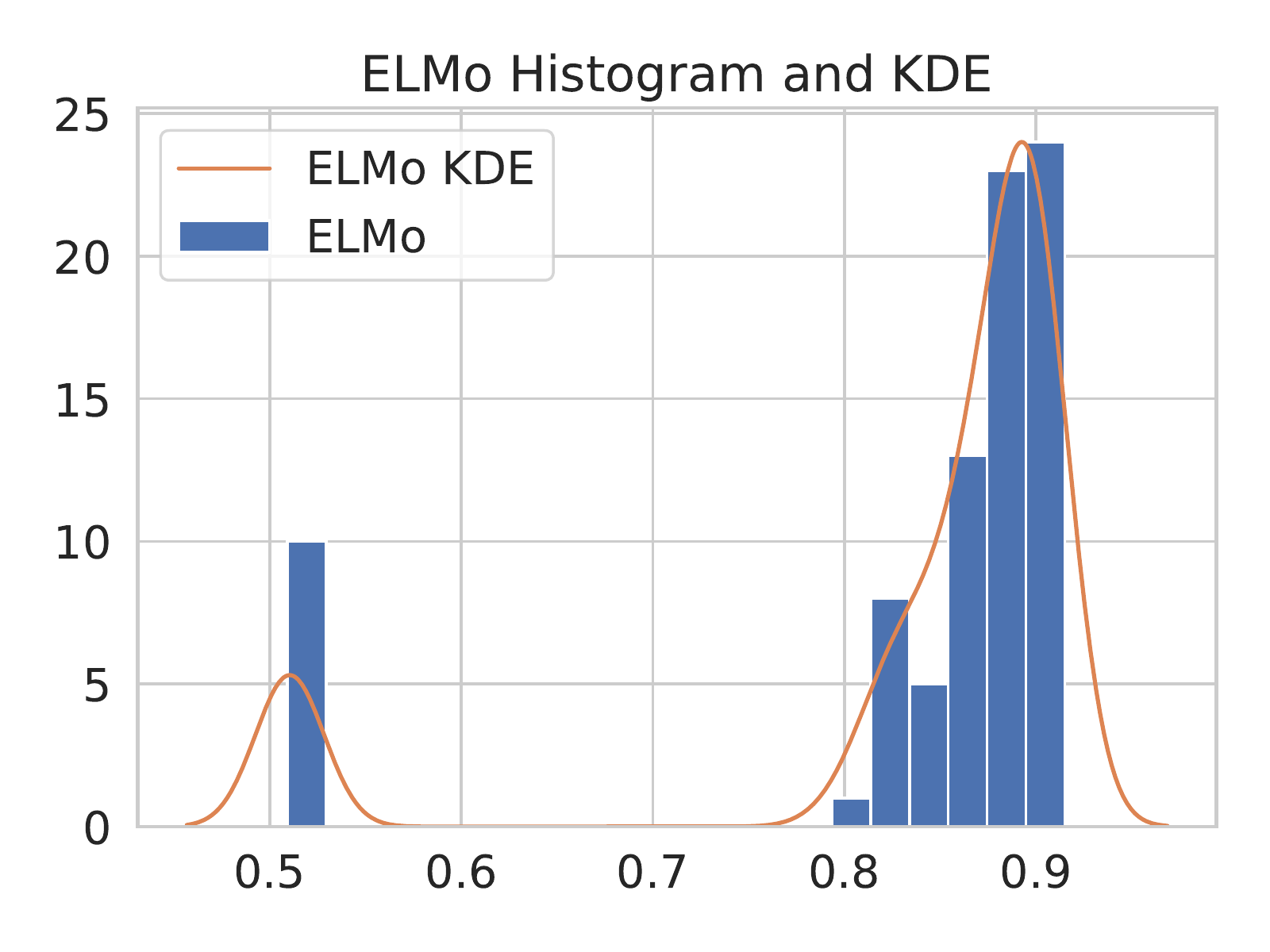}
    \caption{Gaussian kernel density estimators fitted to each model's results, along with the histograms of the original runs.}
    \label{fig:hist}
\end{figure}

We present hyperparameters in Tables~\ref{tab:hyperparams} and \ref{tab:kparams} and Figure~\ref{fig:hist}.
We conduct all GloVe and ELMo experiments using PyTorch 1.3.0 with CUDA 10.0 and cuDNN 7.6.3, running on NVIDIA Titan RTX, Titan V, and RTX 2080 Ti graphics accelerators. Our MLP and LSTM experiments use PyTorch 0.4.1 with CUDA 9.2 and cuDNN 7.1.4, running on RTX 2080 Ti's. We use Hedwig\footnote{\url{https://github.com/castorini/hedwig}} for the document classification experiments and the Show Your Work codebase (see link in Table~\ref{tab:hyperparams}) for the sentiment classification ones.
	
\end{document}